\documentclass[sigconf,natbib=true,nonacm,review=false,screen=true]{acmart}

\acmSubmissionID{65}

\usepackage[inline]{enumitem}
\usepackage{acronym}
\usepackage{xkcdcolors}

\PassOptionsToPackage{dvipsnames}{xcolor}

\usepackage{dsfont}
\usepackage{bbm}
\usepackage{bm}
\usepackage[skip=0pt]{caption}
\usepackage{graphicx}
\usepackage{multirow}
\usepackage{tabularx}
\usepackage[dvipsnames]{xcolor}
\usepackage{numprint}
\usepackage[normalem]{ulem}
\usepackage{multirow}
\usepackage{amsthm}
\usepackage[english]{babel}
\usepackage{amsthm}
\usepackage{mleftright}
\usepackage{graphicx}
\usepackage{adjustbox}
\usepackage{mathtools}

\acrodef{CF}{collaborative filtering}
\acrodef{LTR}{learning to rank}
\acrodef{NDCG}{normalized discounted cumulative gain}
\acrodef{DCG}{discounted cumulative gain}
\acrodef{VAE}{variational autoencoder}
\acrodef{VAE}{variational autoencoder}
\acrodef{ELBO}{evidence lower bound objective}
\acrodef{IPS}{inverse propensity scoring}
\acrodef{BPR}{bayesian personalized ranking}
\acrodef{MF}{matrix factorization}
\acrodef{MNAR}{missing-not-at-random}
\acrodef{ULTR}{unbiased learning-to-rank}
\acrodef{CLTR}{counterfactual learning to rank}
\acrodef{LOLN}{law of large numbers}
\acrodef{CRM}{counterfactual risk minimization}
\acrodef{IS}{importance sampling}
\acrodef{i.i.d}{independent and identically distributed}
\acrodef{CRM}{counterfactual risk minimization}
\acrodef{PL}{Plackett-Luce}
\acrodef{CTR}{click through rate}
\acrodef{SEA}{safe exploration algorithm}
\acrodef{GENSPEC}{generalization and specialization }
\acrodef{VCRM}{variational counterfactual risk minimization}
\acrodef{SGD}{stochastic gradient descent}
\acrodef{DR}{doubly robust}
\acrodef{DM}{direct method}
\acrodef{ERC}{exposure ratio clipping}
\acrodef{PRPO}{proximal ranking policy optimization}
\acrodef{PPO}{proximal policy optimization}
\acrodef{RL}{reinforcement learning}

\theoremstyle{plain}
\newtheorem{theorem}{Theorem}[section]

\theoremstyle{definition}

\theoremstyle{remark}

\newcommand{\headernodot}[1]{\vspace{1mm}\noindent\textbf{#1}}
\newcommand{\header}[1]{\headernodot{#1.}}

\allowdisplaybreaks

\author{Shashank Gupta}
\orcid{0000-0003-1291-7951}
\affiliation{%
  \institution{University of Amsterdam}
  \city{Amsterdam}
  \country{The Netherlands}
}
\email{s.gupta2@uva.nl}

\author{Harrie Oosterhuis}
\orcid{0000-0002-0458-9233}
\affiliation{%
  \institution{Radboud University}
  \city{Nijmegen}
  \country{The Netherlands}
}
\email{harrie.oosterhuis@ru.nl}

\author{Maarten de Rijke}
\orcid{0000-0002-1086-0202}
\affiliation{%
  \institution{University of Amsterdam}
  \city{Amsterdam}
  \country{The Netherlands}
}
\email{m.derijke@uva.nl}

\title[Proximal Ranking Policy Optimization for Practical Safety in Counterfactual Learning to Rank]{Proximal Ranking Policy Optimization for Practical Safety in Counterfactual Learning to Rank (Extended Abstract)}
\titlenote{Based on a full paper titled \emph{``Practical and Robust Safety Guarantees for Advanced Counterfactual Learning to Rank"} published at CIKM'24~\cite{gupta-2024-practical}.}

\allowdisplaybreaks

\begin{document}

% \begin{abstract}
% Safe \ac{CLTR} aims to mitigate the risks of using \acl{IPS} to correct for position bias.
% However, existing safe \ac{CLTR} relies on specific assumptions about user behavior, and thereby only provides a conditional form of safety.
% %
% We propose a novel approach, \acfi{PRPO}, that removes incentives for learning ranking behavior that is too dissimilar to a safe ranking model.
% Thereby, \ac{PRPO} imposes a limit on how much learned models can degrade performance metrics, \emph{without} relying on any specific user assumptions.
% Our experiments show that our novel \ac{PRPO} method provides higher performance than the existing safe \ac{CLTR} approaches.
% %
% % \ac{PRPO} always maintains safety, even in maximally adversarial situations.
% % By avoiding assumptions, \ac{PRPO} is the first method with \emph{unconditional} safety in deployment that translates to robust safety for real-world applications.
% \end{abstract}

\maketitle

\acresetall

% \input{sections/01-introduction}
% \input{sections/02-related-works}
% !TEX root = ../2024-crm-ppo.tex

\section{Introduction: Safe Counterfactual LTR}
The goal in \ac{LTR} is to find a ranking policy ($\pi$) that optimizes a given ranking metric~\cite{liu2009learning}; 
given a set of documents ($D$), a distribution of queries $Q$, and relevance function ($P(R=1 \mid d)$):
\begin{equation}
    U(\pi) =  \sum_{q \in Q} P(q \mid Q) \sum_{d} \omega(d \mid \pi) \; P(R=1 \mid d), \label{true-utility}
\end{equation}
where $\omega(d \mid \pi)$ is the weight of the document for a given policy $\pi$;
e.g., for the \ac{DCG} metric~\citep{jarvelin2002cumulated}:
\begin{equation}
    \omega_{\text{DCG}}(d \mid q, \pi) = \mathbb{E}_{y \sim \pi( \cdot \mid q)} \mleft[ (\log_2(\textrm{rank}(d \mid y) + 1))^{-1} \mright]. \label{rho}
\end{equation}
Because the relevance function is generally unknown, \ac{CLTR} uses clicks to estimate $U(\pi)$~\cite{joachims2017unbiased, wang2016learning, oosterhuis2020learning}.
The  \emph{policy-aware} approach~\citep{oosterhuis2020policy} starts by assuming the rank-based position bias model~\citep{joachims2017unbiased, wang2016learning,craswell2008experimental}, where the probability of a click on document $d$ at position $k$ is the product of a rank-based examination probability and document relevance:
\begin{equation}
P(C = 1 \mid d, k) = P(E=1 \mid k) P(R=1 \mid d) = \alpha_k P(R=1 \mid d).
\label{eq:clickmodel}
\end{equation}
Let $\mathcal{D}$ be a set of logged interaction data:
%
%\begin{equation}
$
    \mathcal{D} = \big\{q_i, c_i \big\}^N_{i=1}$,
%    \label{logs}
%\end{equation}
%
where each of the $N$ interactions consists of a query $q_i$ and click feedback $c_i(d) \in \{0,1\}$.
With the policy-aware propensities: $\rho_{0}(d \mid q_i, \pi_0) =  \mathbb{E}_{y \sim \pi_{0}(q_i)} \big[ \alpha_{k(d)}  \big] = \rho_{i,0}(d)$
%
% \begin{equation}
%         \rho_{0}(d \mid q_i, \pi_0) =  \mathbb{E}_{y \sim \pi_{0}(q_i)} \big[ \alpha_{k(d)}  \big] = \rho_{i,0}(d).
%     \label{policy-aware-exposure}
% \end{equation}
%
and $\omega(d \mid q_i, \pi) = \omega_i(d)$ for brevity,
the policy-aware \ac{IPS} estimator is:
\begin{equation}
    \hat{U}_{\text{IPS}}(\pi) = \frac{1}{N} \sum_{i=1}^{N} \sum_{d}  \frac{\omega_i(d)}{\rho_{i,0}(d)} c_i(d).
    \label{cltr-obj-ips-positionbias}
\end{equation}
While being unbiased, \Ac{IPS} \ac{CLTR} methods can suffer from high-variance~\cite{gupta2024unbiased,oosterhuis2022doubly,joachims2017unbiased}.
Specifically, if data is limited, training an \ac{IPS}-based method can lead to an unreliable and unsafe ranking policy~\cite{gupta2023safe}, a problem also known in the bandit literature~\citep{thomas2015high, jagerman2020safe, swaminathan2015batch, wu2018variance, gupta2024optimal}. 
%  The problem of \textit{safe} policy learning is well-studied in the bandit literature~\citep{thomas2015high, jagerman2020safe, swaminathan2015batch, wu2018variance}. \citet{swaminathan2015batch} proposed the first risk-aware off-policy learning method for bandits, with their risk term quantified as the variance of the \ac{IPS}-estimator. 
% \citet{wu2018variance} proposed an alternative method for risk-aware off-policy learning, where the risk is quantified using a Renyi divergence between the action distribution of the new policy and the logging policy~\citep{renyi1961measures}.
% Thus, both consider it a risk for the new policy to be too dissimilar to the logging policy, which is presumed safe.
% Whilst effective at standard bandit problems, these risk-aware methods are not effective for ranking tasks due to their enormous combinatorial action spaces and correspondingly small propensities.

As a solution for \ac{CLTR}, \citet{gupta2023safe} introduce a risk-aware \ac{CLTR} approach that uses divergence based on the exposure distributions of policies.
They introduce normalized propensities: $\rho'\!(d) = \rho / \sum_{d} \rho(d)$; since $\rho'\!(d) \in [0,1]$ and $\sum_d \rho'\!(d) = 1$, these can be treated as a probability distribution that indicates how exposure is spread over documents.
Subsequently, \citet{gupta2023safe} use Renyi divergence to bound the real utility with probability $1-\delta$:
\begin{equation*}
P\mleft( U(\pi) \!\geq\! \hat{U}_{\text{IPS}}(\pi)\! - \!\!  \sqrt{ \frac{Z}{N}  \Big(\frac{1-\delta}{\delta}\Big) \mathbb{E}_{q} \mleft[ \sum_{d} \mleft(\frac{\rho'(d)}{\rho_{0}'(d)}\mright)^2 \!\! \rho_{0}'(d) \mright]\!}\mright) \!\geq\! 1 - \delta.
\end{equation*}
Recently \citet{gupta-2024-practical} provide a similar bound for doubly-robust \ac{CLTR}~\citep{oosterhuis2022doubly}.
However, these guarantees are only proven to hold under specific click models (e.g., Eq.~\ref{eq:clickmodel}). Therefore, they only provide a \emph{conditional} form of safety that may not hold in practice.

\begin{figure}[!t]
    \begin{center}
    \includegraphics[scale=0.48]{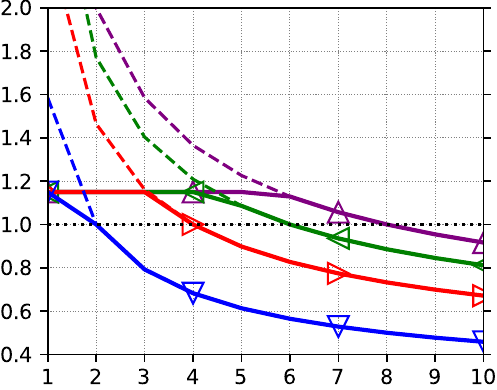}
    \includegraphics[scale=0.48]{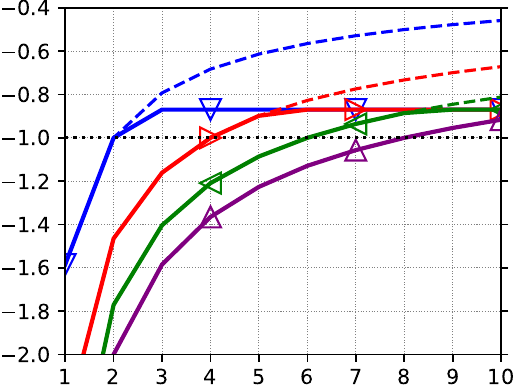}\\
    \includegraphics[width=0.99\columnwidth]{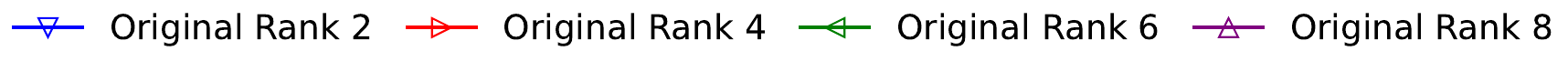}
    \end{center}
        \caption{
        Clipped weight ratios of \acs{PRPO} objective, as documents are moved from four different original ranks.
        Left: positive relevance, $r=1$; right: negative relevance, $r=-1$;
        x-axis: new rank for document;
        y-axis: unclipped weight ratios (dashed lines), $r\cdot\omega_i(d)/\omega_{i,0}(d)$;
        and 
        clipped \acs{PRPO} weight ratios (solid lines),
        $f\mleft(\omega_i(d)/\omega_{i,0}(d), \epsilon_{-} = 1.15^{-1}, \epsilon_{+}= 1.15, r=\pm1\mright)$.
        DCG metric weights used: $\omega_i(d) = \log_2(\textmd{rank}(d \mid q_i, \pi) + 1)^{-1}$.
        }
        \label{fig:prpo}
        \vspace{-\baselineskip}
\end{figure}

{\renewcommand{\arraystretch}{0.01}
\setlength{\tabcolsep}{0.04cm}
\begin{figure*}[ht!]
\vspace{-\baselineskip}
\centering
\begin{tabular}{c r r r }
&
 \multicolumn{1}{c}{ \small \hspace{0.5cm} Yahoo! Webscope}
&
 \multicolumn{1}{c}{ \small \hspace{0.5cm} MSLR-WEB30k}
&
 \multicolumn{1}{c}{ \small \hspace{0.5cm} Istella}
\\
\rotatebox[origin=lt]{90}{\hspace{1.4cm}\small  NDCG@5} &
\includegraphics[scale=0.455]{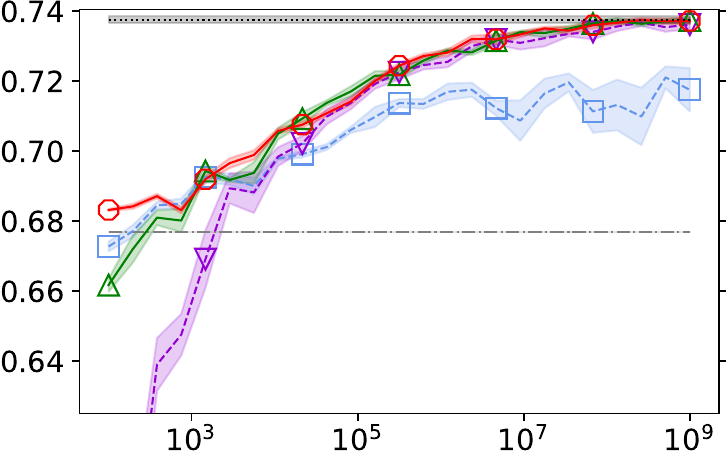} &
\includegraphics[scale=0.455]{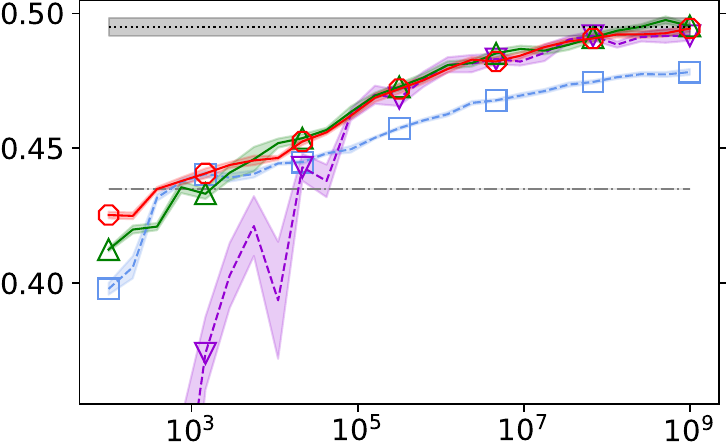} &
\includegraphics[scale=0.455]{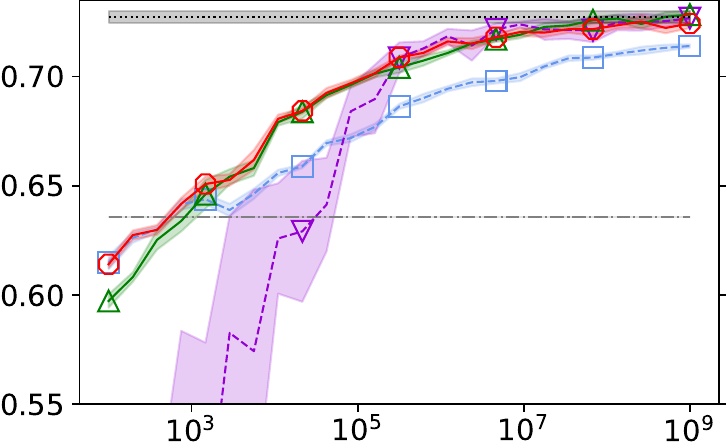}
\\
& \multicolumn{1}{c}{\small \hspace{1.75em} Number of interactions simulated ($N$)}
& \multicolumn{1}{c}{\small \hspace{1.75em} Number of interactions simulated ($N$)}
& \multicolumn{1}{c}{\small \hspace{1.75em} Number of interactions simulated ($N$)}
\\[2mm]
    \multicolumn{4}{c}{
    \includegraphics[scale=0.4]{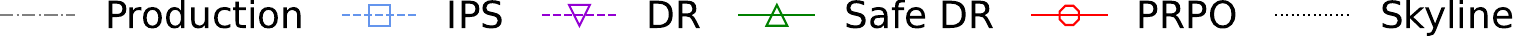}
}
\end{tabular}
%\vspace{0.1\baselineskip}
\caption{
Performance in terms of NDCG@5 of the \ac{IPS}, \ac{DR}, safe \ac{DR} ($\delta=0.95$) and \ac{PRPO} ($\delta(N)=\frac{100}{N}$), with the number of simulated queries in the training data ($N$) varying from $10^2$ to $10^9$.
% Results are averaged over 10 runs; the shaded areas indicate 80\% prediction intervals. 
}
\label{fig:mainresults}
\end{figure*}
}

\setlength{\tabcolsep}{0.15em}
{\renewcommand{\arraystretch}{0.63}
\begin{figure*}[th!]
\vspace{-0.5\baselineskip}
\centering
\begin{tabular}{c r r r r}
&
 \multicolumn{1}{c}{  Yahoo! Webscope}
&
 \multicolumn{1}{c}{  MSLR-WEB30k}
&
 \multicolumn{1}{c}{  Istella}
 &

\\
\rotatebox[origin=lt]{90}{\hspace{0.65cm} \small NDCG@5} &
\includegraphics[scale=0.455]{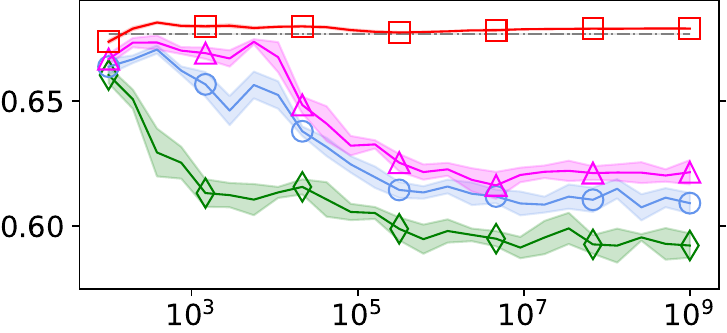} &
\includegraphics[scale=0.455]{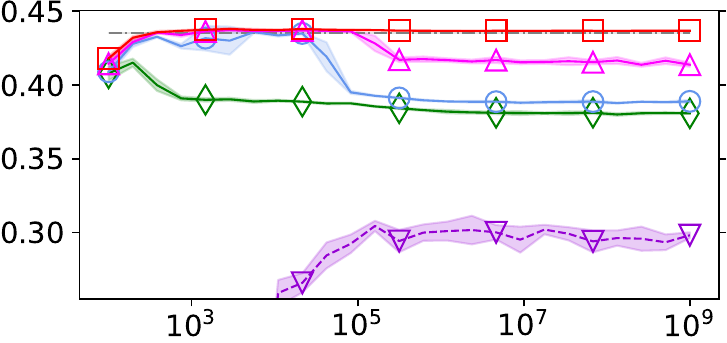} &
\includegraphics[scale=0.455]{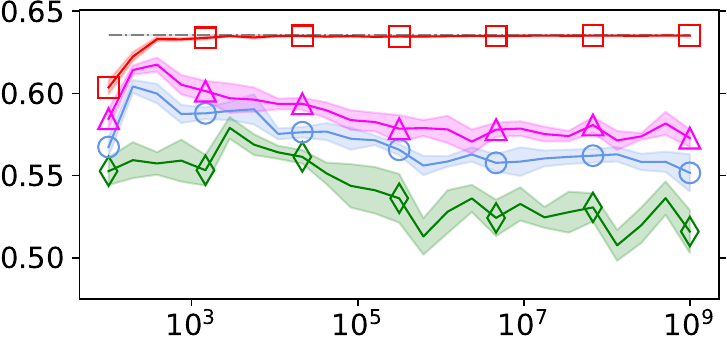} 
\\
\multicolumn{4}{c}{
\includegraphics[scale=.4]{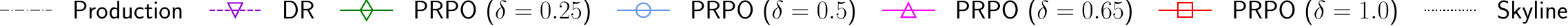}
}
\\
%\cmidrule{2-4}
%\\
\rotatebox[origin=lt]{90}{\hspace{0.65cm} \small NDCG@5} &
\includegraphics[scale=0.455]{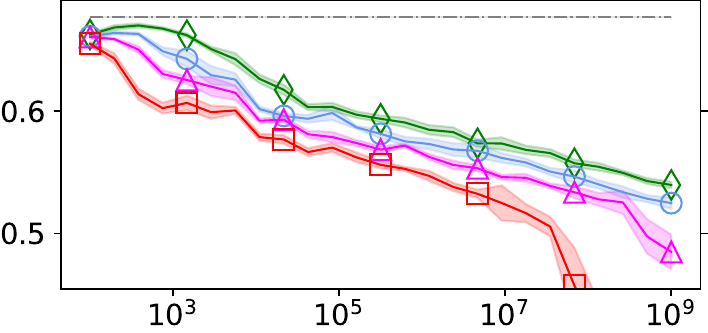} &
\includegraphics[scale=0.455]{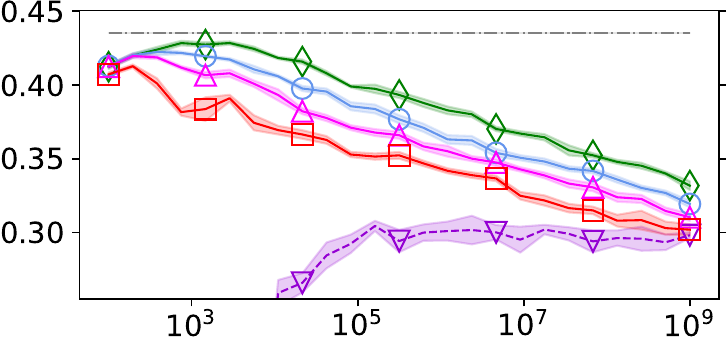} &
\includegraphics[scale=0.455]{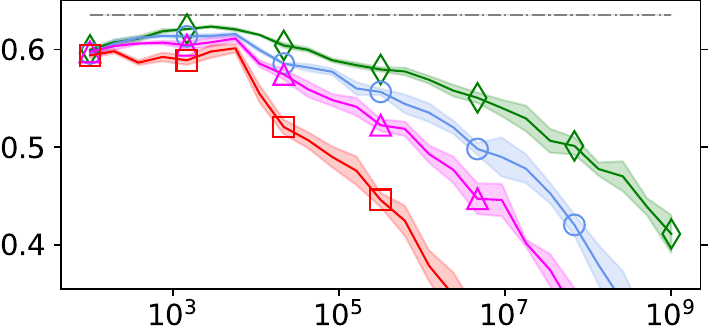}
\\
& \multicolumn{1}{c}{\small \hspace{1.75em} Number of interactions simulated ($N$)}
& \multicolumn{1}{c}{\small \hspace{1.75em} Number of interactions simulated ($N$)}
& \multicolumn{1}{c}{\small \hspace{1.75em} Number of interactions simulated ($N$)}
\\[2mm]
\multicolumn{4}{c}{
\includegraphics[scale=.4]{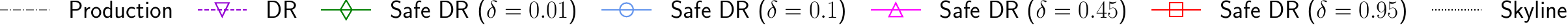}
}
\end{tabular}
%\vspace{0.1\baselineskip}
\caption{
    Performance of safe \ac{DR} and \ac{PRPO} under an adversarial click model for varying data sizes. 
    % Top: sensitivity analysis results for the \ac{PRPO} method with varying clipping parameter ($\delta$). 
    % Bottom: sensitivity analysis for the safe \ac{DR} method with varying safety confidence parameter ($\delta$). 
    % Results are averaged over 10 independent runs; the shaded areas indicate $80\%$ prediction intervals.
}
\label{fig:ablationresults_adv}
\vspace{-1.2\baselineskip}
\end{figure*}
}

\section{Proximal Ranking Policy Optimization}

In the field of \acf{RL}, \acfi{PPO} has been introduced as a method to restrict a new policy $\pi$ from deviating too much from a previously rolled-out policy $\pi_0$~\cite{schulman2015high, schulman2017proximal}.
In contrast with the earlier discussed methods, \ac{PPO} does not make use of a divergence term but uses a simple clipping operation in its optimization objective.
Inspired by this approach, we propose the first \emph{unconditionally} safe \ac{CLTR} method: \acfi{PRPO}.
Our method is designed for practical safety by making \emph{no assumptions} about user behavior.
Thereby, it provides the most robust safety for \ac{CLTR} yet.

For safety, instead of relying on a high-confidence bound~\cite{gupta-2024-practical,gupta2023safe}, \ac{PRPO} guarantees safety by removing the incentive for the new policy to rank documents too much higher than the safe logging policy. 
This is achieved by directly clipping the ratio of the metric weights for a given query $q_i$ under the new policy $\omega_i(d)$, and the logging policy ($\omega_{i,0}(d)$), i.e., $\frac{\omega_{i}(d)}{\omega_{i,0}(d)}$ to be bounded in a fixed predefined range: $\mleft[\epsilon_{-}, \epsilon_{+}\mright]$. 
As a result, the \ac{PRPO} objective provides no incentive for the new policy to produce weights $\omega_i(d)$ outside of the range: $\epsilon_{-} \cdot \omega_{i,0}(d) \leq \omega_i(d) \leq \epsilon_{+} \cdot \omega_{i,0}(d)$.
The first step to applying \ac{PRPO} is rewriting the utility of a ranking policy as a sum over the ratios multiplied by a reward function; this can be done without loss of generality.
For example, for the \ac{IPS} estimator:
\begin{equation*}
    \hat{U}(\pi) = \frac{1}{N} \sum_{i=1}^N  \sum_{d}  \frac{\omega_i(d)}{\omega_{i,0}(d)} r(d |\, q_i),
    \quad
    r_\text{IPS}(d |\, q_i) = \frac{\omega_{i,0}(d |\, q)}{\rho_{i,0}(d)}c_i(d)
    .
    \label{eq:estimator_reformulate}
\end{equation*}
\ac{PRPO} applies the following clipping function to this estimator:
\begin{equation}
    f(x,\epsilon_{-}, \epsilon_{+}, r) = 
\begin{cases}
    \min(x, \epsilon_{+}) \cdot r  & r \geq 0,\\
    \max(x, \epsilon_{-}) \cdot r & \text{otherwise}.
\end{cases}  
\label{eq:clipping_function}  
\end{equation}
%
% We note that the direction of clipping depends on the sign of $r$ value.
Applying the clipping function results in the \ac{PRPO} estimator:
\begin{equation}
    \hat{U}_{\text{PRPO}}(\pi) =  \sum_{q,d \in \mathcal{D}}  f\mleft(\frac{\omega(d \mid q)}{\omega_{0}(d \mid q)}, \epsilon_{-}, \epsilon_{+}, r(d \mid q)\mright).
    \label{eq:prpo_obj}
\end{equation}
Additionally, we propose adaptive clipping that loosens as the number of datapoints $N$ increases, e.g., through a monotonically decreasing function $\delta(N)$: $\lim_{N \rightarrow \infty} \delta(N) = 0$; $\epsilon_{-} = \delta(N)$ and $\epsilon_{+} = \frac{1}{\delta(N)}$.

Figure~\ref{fig:prpo} provides a visualization of clipping that shows how it removes incentives for optimization to move documents too far from their original rank.
Appendix~\ref{sec:prpo-proof} provides a theoretical analysis.

In summary, \ac{PRPO} provides safety in \ac{CLTR} through clipping, without relying on assumptions about user behavior.
This makes it uniquely robust and practical for realistic ranking settings.

% !TEX root = ../2024-crm-ppo.tex

% \setlength{\tabcolsep}{0.15em}

\section{Results and Discussion}

For our experiments, we apply a semi-synthetic experimental setup common in the \ac{CLTR} literature~\citep{oosterhuis2022doubly,oosterhuis2021robust,vardasbi2020inverse,gupta2023safe} on the three largest publicly available \ac{LTR} datasets: Yahoo!\ Webscope~\cite{chapelle2011yahoo}, Istella~\citep{dato2016fast}, and MSLR-WEB30k~\citep{qin2013introducing}.
See Appendix~\ref{sec:experimentalsetup} for details of this setup.

Figure~\ref{fig:mainresults} shows the performance of \ac{PRPO} (applied to a \ac{DR} estimator~\citep{gupta-2024-practical, oosterhuis2022doubly}) and our baselines: \ac{IPS} and \ac{DR} estimation and their safe \ac{CLTR} counterparts~\citep{gupta-2024-practical, gupta2023safe}, in a setting where assumptions about user behavior are correct.
We see that only safe \ac{DR} and \ac{PRPO} have the best convergence without an initial period of poor performance.
In contrast, plain \ac{DR} has poor performance when $N < 10^5$ but excellent performance as $N>10^6$; whereas \ac{IPS} does not have poor performance when $N$ is small but is outperformed by all other methods when $N > 10^4$.
Because the performance differences between safe \ac{DR} and \ac{PRPO} appear marginal, we conclude that -- in this setting where correct user behavior is assumed -- both methods provide comparable performance.
Given the simplicity and lack of assumptions of \ac{PRPO}, we see this as a positive result.

Figure~\ref{fig:ablationresults_adv} shows the performance of \ac{DR}, safe \ac{DR}, and \ac{PRPO} (with a static clipping parameter $\delta$) in an adversarial setting where the assumptions about user behavior are catastrophically incorrect.
We see that all methods result in a degradation of performance compared to the production policy.
Moreover, the performance of (unsafe) \ac{DR} is so poor that it mostly falls outside the region visualized by the graphs.
For safe \ac{DR}, we see that performance continues to decline as more data is available (and $N$ increases), a trend that continues regardless of the $\delta$ parameter.
Thereby, it appears that safe \ac{DR} is unable to provide safety in this setting where its assumptions are incorrect.
In contrast, the decreases in performance of \ac{PRPO} are clearly bounded and appear stable once $N > 10^6$.
Thus -- even in this worst case scenario -- \ac{PRPO} successfully bounds the damage that \ac{CLTR} can do to ranking performance.

We conclude that \ac{PRPO} provides a robust and practical form of safety that encompasses an important progression for safe \ac{CLTR}.
Appendix~\ref{sec:extensions} discusses further implications for the general \ac{LTR} field.

% !TEX root = ../2024-crm-ppo.tex

% !TEX root = ../2024-crm-ppo.tex
% \vspace{-5.5mm}

% \section*{Acknowledgements}
% This research was supported by Huawei Finland, 
% the Dutch Research Council (NWO), under project numbers VI.Veni.222.269, 024.004.022, NWA.1389.20.\-183, and KICH3.LTP.20.006, 
% the EU's Horizon Europe program under grant No 101070212,
% and
% a SURF Cooperative using grant no.\ EINF-8200.
% All content represents the opinion of the authors, which is not necessarily shared or endorsed by their respective employers and/or sponsors.

\section*{Acknowledgements}
This research was supported by Huawei Finland, 
the Dutch Research Council (NWO), under project numbers VI.Veni.222.269, 024.004.022, NWA.1389.20.\-183, and KICH3.LTP.20.006, 
the EU's Horizon Europe program under grant No 101070212,
and
a SURF Cooperative using grant no.\ EINF-8200.
All content represents the opinion of the authors, which is not necessarily shared or endorsed by their respective employers and/or sponsors.

\balance
\bibliographystyle{ACM-Reference-Format}
\bibliography{bibliography}

\appendix
\section*{Appendix}
% !TEX root = ../2024-crm-ppo.tex

%\vspace*{-4.8mm}

\section{Simulation and Evaluation Setup}
\label{sec:experimentalsetup}

We use the three largest publicly available \ac{LTR} datasets: Yahoo!\ Webscope~\cite{chapelle2011yahoo}, MSLR-WEB30k~\citep{qin2013introducing}, and Istella~\citep{dato2016fast}. The datasets consist of queries, a preselected list of documents per query, query-document feature vectors, and manually-graded relevance judgments for each query-document pair. Following~\cite{oosterhuis2022doubly,vardasbi2020inverse,gupta2023safe}, we train a production ranker on a $3\%$ fraction of the training queries and their corresponding relevance judgments.
The goal is to simulate a real-world setting where a ranker trained on manual judgments is deployed in production and is used to collect click logs.
The collected click logs can then be used for \ac{CLTR}. 
We assume the production ranker is safe, given that it would serve live traffic in a real-world setup.

We simulate a top-$K$ ranking setup~\cite{oosterhuis2020policy} where only $K=5$ documents are displayed to the user for a given query, and any document beyond that gets zero exposure.
To get the relevance probability, we apply the following transformation: $P(R=1 \mid q, d) = 0.25 \cdot rel(q,d)$, where $rel(q,d) \in \{0,1,2,3,4\}$ is the relevance judgment for the given query-document pair.
We generate clicks based on the trust bias click model:
\begin{equation}
    P(C=1 \mid q, d, k) = \alpha_k P(R=1 \mid q, d) + \beta_k.   
\label{click_simulation}
\end{equation}
The trust bias parameters are set based on the empirical observation in \citep{agarwal2019addressing}: $\alpha = [0.35, 0.53, 0.55, 0.54, 0.52]$, and $\beta = [0.65, 0.26, 0.15$, $0.11$, $0.08]$.
For \ac{CLTR} training, we only use the training and validation clicks generated via the click simulation process (Eq.~\ref{click_simulation}).
To test the robustness of the safe \ac{CLTR} methods in a setting where the click model assumptions do not hold,
we simulate an \emph{adversarial click model}, where the user clicks on the irrelevant document with a high probability and on a relevant document with a low click probability.
We define the adversarial click model as:
\begin{equation}
    P(C=1 \mid q, d, k) = 1 - \mleft( \alpha_k P(R=1 \mid q, d) + \beta_k \mright).   
\label{click_simulation_adv}
\end{equation}
Thereby, we simulate a maximally \emph{adversarial} user who clicks on documents with a click probability that is inversely correlated with the assumed trust bias model.

Further, we assume that the logging propensities have to be estimated.
For the logging propensities $\rho_0$, and the logging metric weights ($\omega_0$), we use a simple Monte-Carlo estimate~\citep{gupta2023safe}:
\begin{equation}
    \hat{\rho}_0(d ) = \frac{1}{N} \sum^N_{i=1: y_i \sim \pi_{0}\hspace{-1cm}}   \alpha_{k_{i}(d)},
    \quad
    \hat{\omega}_0(d ) = \frac{1}{N} \sum^N_{i=1: y_i \sim \pi_{0}\hspace{-0.8cm}} \mleft( \alpha_{k_{i}(d)} + \beta_{k_{i}(d)} \mright).
    \label{prop-estimate}  
\end{equation}
%
% As is common in \ac{CLTR}~\citep{oosterhuis2020learning, saito2021counterfactual, joachims2017unbiased},    
For the learned policies ($\pi$), we optimize \ac{PL} ranking models~\citep{oosterhuis2021computationally} using the REINFORCE policy-gradient method~\cite{yadav2021policy,gupta2023safe}.
We perform clipping on the logging propensities only for the training clicks and not for the validation set. 
Following previous work, we set the clipping parameter to $10 / \sqrt{N}$~\cite{gupta2023safe,oosterhuis2021unifying}. 
We do not apply the clipping operation for the logging metric weights.
To prevent overfitting, we apply early stopping based on the validation clicks.
For variance reduction, we follow~\cite{yadav2021policy,gupta2023safe} and use the average reward per query as a control-variate.

As our evaluation metric, we compute the NDCG@5 metric using the relevance judgments on the test split of each dataset~\citep{jarvelin2002cumulated}. 
Finally, the following methods are included in our comparisons:
\begin{enumerate*}[leftmargin=*,label=(\roman*)]
    \item  \emph{IPS}. The \ac{IPS} estimator with affine correction~\cite{vardasbi2020inverse,oosterhuis2021unifying} for \ac{CLTR} with trust bias.
    \item  \emph{Doubly Robust}. The \ac{DR} estimator for \ac{CLTR} with trust bias~\cite{oosterhuis2022doubly}. 
    % This is the most important baseline for this work, given that the \ac{DR} estimator is the state-of-the-art \ac{CLTR} method~\cite{oosterhuis2022doubly}.
    %
    \item  \emph{Safe \ac{DR}}. The recently-introduced safe \ac{DR} \ac{CLTR} method~\citep{gupta-2024-practical}, which relies on a trust bias assumption.
     \item  \emph{\ac{PRPO}}. Our proposed \acfi{PRPO} method for safe \ac{CLTR} (Eq.~\ref{eq:prpo_obj}), we apply it to the \ac{DR} estimator (see \citep{gupta-2024-practical} for a more detailed formulation).
     \item  \emph{Skyline.} \ac{LTR} method trained on the true relevance labels. Given that it is trained on the real relevance signal, the skyline performance is the upper bound on possible \ac{CLTR} performance. 
\end{enumerate*}
All source code to reproduce our experimental results is available at: 
\url{https://github.com/shashankg7/cikm-safeultr}.

\section{Formal Proof of Safety}\label{sec:prpo-proof}
\begin{theorem}
Let $q$ be a query, $\omega$ be metric weights, $y_0$ be a logging policy ranking, and $y^*(\epsilon_{-},\epsilon_{+})$ be the ranking that optimizes the \ac{PRPO} objective in Eq.~\ref{eq:prpo_obj}.
Assume that $\forall d, \in \mathcal{D}, r(d \mid q) \not= 0$.
Then, 
%For any query $q_i$, any choice of metric weights $\omega$, and any logging policy ranking $y_0$, let $y^*(\epsilon_{-},\epsilon_{+})$ be the ranking that optimizes the \ac{PRPO} objective in Eq.~\ref{eq:prpo_obj} and $\forall q,d, \in \mathcal{D}, r(d \mid q) \not= 0$,
for any $\Delta \in \mathbb{R}_{\geq0}$, there exist values for $\epsilon_{-}$ and $\epsilon_{+}$ that guarantee that the difference between the utility of $y_0$ and $y^*(\epsilon_{-},\epsilon_{+})$ is bounded by $\Delta$:
\begin{equation}
\forall \Delta \! \in \mathbb{R}_{\geq0}, \exists \epsilon_{-} \!\!\in \mathbb{R}_{\geq0}, \epsilon_{+}\!\! \in \mathbb{R}_{\geq0};  \;| U(y_0) -  U(y^*(\epsilon_{-}, \epsilon_{+}))  | \leq \Delta.
\label{eq:prpotheorem}
\end{equation}
\end{theorem}
\begin{proof}
    Given a logging policy ranking $y_0$, a user defined metric weight $\omega$, and non-zero $r(d \mid q)$, for the choice of the clipping parameters $\epsilon_{-} = \epsilon_{+} = 1$, 
    the ranking $y^*(\epsilon_{-},\epsilon_{+})$ that maximizes the \ac{PRPO} objective (Eq.~\ref{eq:prpo_obj}) will be the same as the logging ranking $y_0$, i.e. $y^*(\epsilon_{-},\epsilon_{+})=y_0$.
    This is trivial to prove since any change in ranking can only lead to a decrease in the clipped ratio weights, and thus, a decrease in the \ac{PRPO} objective.
    Therefore, $y^*(\epsilon_{-}=1,\epsilon_{+}=1)=y_0$ when $\epsilon_{-} = \epsilon_{+} = 1$.
    Accordingly: $| U(y_0) -  U(y^*(\epsilon_{-}=1, \epsilon_{+}=1)) | = 0$ directly implies Eq.~\ref{eq:prpotheorem}.
    This completes our proof.
\end{proof}

\noindent%
Whilst the above proof is performed through the extreme case where $\epsilon_{-} = \epsilon_{+} = 1$ and the optimal ranking has the same utility as the logging policy ranking,
other choices of $\epsilon_{-}$ and $\epsilon_{+}$ bound the difference in utility to a lesser degree and allow for more deviation.
As our experimental results show, the power of PRPO is that it gives practitioners direct control over this maximum deviation.

% \section{Visualization of PRPO clipping}
% Fig.~\ref{fig:prpo} visualizes the effect the clipping of \ac{PRPO} has on the optimization incentives.
% We see how the clipped and unclipped weight ratios progress as documents are placed on different ranks.
% The unclipped weights keep increasing as documents are moved to the top of the ranking, when $r>1$, or to the bottom, when $r<1$.
% Consequently, optimization with unclipped weight ratios aims to place these documents at the absolute top or bottom positions.
% Conversely, the clipped weights do not increase beyond their clipping threshold, which for most document is reached before being placed at the very top or bottom position.
% As a result, optimization with clipped weight ratios will not push these documents beyond these points in the ranking.
% For example, when $r>0$, we see that there is no incentive to place a document at higher than rank 6, if it was placed at rank 8 by the logging policy.
% Similarly, placement higher than rank 4 leads to no gain if the original rank was 6, and higher than rank 3 leads to no improvement gain from an original rank of 4.
% Vice versa, when $r<0$, each document has a rank, where placing it lower than that rank brings no increase in clipped weight ratio.
% Importantly, this behavior only depends on the metric and the logging policy; \ac{PRPO} makes \emph{no further assumptions}.

\section{Gradient Ascent and Possible Extensions}
\label{sec:extensions}

Finally, we consider how the \ac{PRPO} objective should be optimized. This turns out to be very straightforward when we look at its gradient.
The clipping function $f$ (Eq.~\ref{eq:clipping_function}) has a simple gradient involving an indicator function on whether $x$ is inside the bounds:
\begin{equation}
\nabla_{x} f(x, \epsilon_{-}, \epsilon_{+}, r)
=  \mathds{1}\big[
(r > 0 \land x \leq \epsilon_{+})
\lor
(r < 0 \land x \geq \epsilon_{-})
\big] r.
\end{equation}
Applying the chain rule to the \ac{PRPO} objective (Eq.~\ref{eq:prpo_obj}) reveals:
\begin{equation*}
\nabla_{\!\pi} \hat{U}_{\text{PRPO}}(\pi) \!= 
	\!\!\sum_{q,d \in \mathcal{D}}
	\underbrace{\!\!
	\Big[\nabla_{\!\pi} \frac{\omega(d | q)}{\omega_{0}(d | q)} \Big]
	}_\text{\hspace{-.5cm}grad. for single doc.\hspace{-.5cm}}
	\underbrace{
	\nabla_{\!\pi} f\bigg( \! \frac{\omega(d | q)}{\omega_{0}(d | q)}, \epsilon_{-}, \epsilon_{+}, r(d | q) \! \bigg)
	}_\text{\raisebox{0mm}[1.58mm][0mm]{indicator reward function}}
	.
\end{equation*}
Thus, the gradient of \ac{PRPO} simply takes the importance weighted metric gradient per document, and multiplies it with the indicator function and reward.
As a result, \ac{PRPO} is simple to combine with existing \ac{LTR} algorithms, especially ones that use policy-gradients~\cite{williams1992simple}, such as PL-Rank~\citep{oosterhuis2021computationally, oosterhuis2022learning} or StochasticRank~\citep{ustimenko2020stochasticrank}.
For methods in the family of LambdaRank~\citep{wang2018lambdaloss, burges2010ranknet, burges2006learning}, it is a matter of replacing the $|\Delta DCG|$ term with an equivalent for the PRPO bounded metric.

Lastly, we note that whilst we introduced \ac{PRPO} for \ac{DR} estimation, it can be extended to virtually any relevance estimation by choosing a different $r$;
e.g., one can easily adapt it for relevance estimates from a click model~\citep{chuklin-click-2015}, etc.
In this sense, we argue that \ac{PRPO} can be seen as a framework for robust safety in \ac{LTR} in general.

\end{document}